\documentclass[transaction]{IEEEtran}

\usepackage[nottoc]{tocbibind}

\usepackage[numbers, sort&compress]{natbib} 

\usepackage{graphicx} 
\usepackage[table,xcdraw]{xcolor}
\usepackage{amsmath,amsfonts}
\usepackage{amssymb}
\usepackage{amsthm}
\usepackage{wrapfig}
\usepackage{lipsum,multicol}
\usepackage{varwidth}
\usepackage{tasks}

\graphicspath{{./Experiments/}}

\usepackage{float}
\usepackage{caption}

\usepackage{algorithmic}
\usepackage{algorithm}
\usepackage{array}
\usepackage[caption=false,font=normalsize,labelfont=sf,textfont=sf]{subfig}
\usepackage{textcomp}
\usepackage{stfloats}
\usepackage{verbatim}
\usepackage{setspace}
\usepackage{listings}
\usepackage{xcolor}
\definecolor{codegreen}{rgb}{0,0.6,0}
\definecolor{codegray}{rgb}{0.5,0.5,0.5}
\definecolor{codepurple}{rgb}{0.58,0,0.82}
\definecolor{backcolour}{rgb}{0.95,0.95,0.92}
\lstdefinestyle{codestyle}{
    backgroundcolor=\color{backcolour},   
    commentstyle=\color{codegreen},
    keywordstyle=\color{magenta},
    numberstyle=\tiny\color{codegray},
    stringstyle=\color{codepurple},
    basicstyle=\ttfamily\footnotesize,
    breakatwhitespace=false,         
    breaklines=true,                 
    captionpos=b,                    
    keepspaces=true,                 
    numbers=left,                    
    numbersep=5pt,                  
    showspaces=false,                
    showstringspaces=false,
    showtabs=false,                  
    tabsize=2
}
\usepackage{hyperref}
\usepackage{orcidlink}

\newtheorem{definition}{Definition}			
\newtheorem{proposition}{Proposition}		
\newtheorem{remark}{Remark} 				
\newtheorem{axiom}{Axiom}					
\newtheorem{assumption}{Assumption}			
\newtheorem*{rationale}{Rationale}			

\newcommand{\Csoftmax}{\operatorname{C\text{-}softmax}} 
\newcommand{\Trxd}{\mathcal{T}}
\newcommand{\vecomega}{\boldsymbol{\omega}}		
\newcommand{\vecalpha}{\boldsymbol{\alpha}}		
\newcommand{\vecbeta}{\boldsymbol{\beta}}		

\newcommand{\real}{\mathbb{R}}							
\newcommand{\Raw}{\boldsymbol{R}}						
\newcommand{\Context}{\boldsymbol{C}} 					
\newcommand{\Priority}{\boldsymbol{\Omega}} 			
\newcommand{\Data}{\real^{r \times d}} 					
\newcommand{\Output}{\mathcal{Y}} 						
\newcommand{\Simplex}{\Delta^{d - 1}} 					
\newcommand{\RSpace}{\mathcal{S}} 						
\newcommand{\Act}{\mathcal{G}} 							
\newcommand{\RISimplex}{{\Simplex}^{(r)}_{\text{int}}}	
\newcommand{\DataEpsilon}{\Data_{\epsilon}}				
\newcommand{\DataPos}{(\real_{>0})^{r \times d}}		
 
\newcommand{\GitHubLink}[1]{\href{https://github.com/AIntelligent/Conflict-Framework/releases}{\text{#1}}}
\newcommand{\ZenodoLink}{\href{https://doi.org/10.5281/zenodo.20309044}{\texttt{DOI:10.5281/zenodo.20309044}}}

\newcommand{\CAHP}{\operatorname{C\text{-}AHP}}
\newcommand{\CAHPPP}{\operatorname{C\text{-}AHP\text{++}}}

\usepackage{enumerate}
\usepackage{datetime}
\usepackage{adjustbox}
\usepackage{booktabs}

\begin{document}

	\title{A Mathematical Conflict Framework for Contextual Data Modulation}

	\author{Hakan Emre Kartal,~\IEEEmembership{Member,~IEEE}%
	\orcidlink{0000-0002-3952-7235}}

	\maketitle

	\begin{abstract}
	In this study, a generalized operator-based mathematical conflict framework is presented to explicitly 
	represent structural discrepancies between raw data and contextual data.
	The proposed structure treats conflict as a local, directional, and context-sensitive quantity, integrating
	components such as weighting, scale behavior, and output mapping under a unified abstract operator.
	Without being reduced to a specific learning algorithm or optimization method, the framework is defined as a 
	general structure adaptable to different classes of problems.
	While existing approaches typically treat conflict merely as an implicit side effect embedded within the 
	optimization process, the proposed framework considers conflict as an independent, operator-based, and 
	component-level mathematical object.
	\end{abstract}

	\begin{IEEEkeywords}
		Conflict modeling, 
		Context-aware systems, 
		Priority-Modulated Comparison, 
		Multi-Criteria Decision Making (MCDM), 
		Information Fusion,
		Contextual Conflict Frameworks,
		Decision theory, 
		Conflict operators, 
		Contextual weighting, 
		Representation learning, 
		Signed Discrepancy Measures, 
		Fusion Frameworks, 
		Scale-Invariant Operators
	\end{IEEEkeywords}

\section{Introduction}
\IEEEPARstart{C}{onflict}, namely the discrepancies between raw data and contextual data, has often been studied implicitly in machine learning, multi-criteria decision making (MCDM), and contextual learning problems.
Such discrepancies have frequently been addressed in the literature under concepts such as \textit{loss conflict}
\cite{yu2020gradientsurgerymultitasklearning}, \textit{gradient conflict}
\cite{senerkoltun2019}, \textit{preference inconsistency}
\cite{keeneyraiffa1979}, or \textit{trade-off}.

However, these approaches generally:
\begin{enumerate}
	\item treat conflict as an \textbf{implicit} side effect,
	\item fail to simultaneously model the \textbf{magnitude, directionality, and contextual dependency} of conflict,
	\item do not provide a \textbf{general operator framework} between data representation and decision output.
\end{enumerate}

In this study, conflict is treated not as a classical optimization loss or distance measure, but rather as an interpretation-independent and modular \textbf{operator structure} defined between data representations.

Rather than proposing a new conflict metric, the proposed structure introduces an axiomatic and generalizable mathematical \textbf{Conflict Framework} for the notion of ``conflict.''

This study presents a general framework that defines the notion of conflict, representing directional discrepancies between raw values and contextual values, within an axiomatic structure.
Accordingly, it does not propose a new distance measure, divergence, or optimization technique.

In this context, the proposed approach should not be interpreted as a direct generalization of classical divergence or distance measures.
In particular, while structures such as ``Bregman divergences'' and the ``$f$-divergence family''
\cite{bregman1967relaxation, csiszar1967information} aim to quantitatively characterize separations between distributions or measures, the notion of conflict considered in this study is treated as a \textbf{directional, contextually modulated, and operator-based structural discrepancy object}.

The proposed framework consists of three components:
\begin{enumerate}[i]
	\item a conflict kernel $\Phi^{(g)}$ that measures directional discrepancies between components
	(see Definition~\ref{def:conflict-kernel-operator}),
	\item a priority matrix $\Priority$
	that represents relative priorities among criteria
	(see Definition~\ref{def:priority-structual-matrix}), and
	\item a projection operator $\Psi$ that maps conflict outputs into a decision or analysis space
	(see Definition~\ref{def:fusion-operator}).
\end{enumerate}

This separation enables conflict measurement, contextual weighting, and decision reduction processes to be structurally disentangled, yielding a modular and extensible framework.

Existing approaches generally treat conflict as an implicit side effect embedded within the optimization process, as a symmetric divergence, or as an inter-task balancing problem.
Although these approaches produce effective results within their respective problem domains, they do not structurally model the directional nature of conflict, contextual priority modulation, and decision projection as mutually separated components.

In contrast, the proposed framework treats conflict not as a byproduct of a specific optimization function, but as an independent mathematical object with explicitly separated layers of measurement, modulation, and projection.

\begin{table}[!ht]
	\resizebox{\columnwidth}{!}{%
	\begin{tabular}{|l|c|c|c|}
		\hline
		\rowcolor[HTML]{C0C0C0} 
		\textbf{Property} & \textbf{\begin{tabular}[c]{@{}c@{}}Loss-Based\\ Approaches\end{tabular}} & \textbf{\begin{tabular}[c]{@{}c@{}}Divergence/Distance\\ Approaches\end{tabular}} & \textbf{\begin{tabular}[c]{@{}c@{}}Proposed\\ Framework\end{tabular}} \\ \hline
		\begin{tabular}[c]{@{}l@{}}Explicit representation\\ of conflict\end{tabular} & Implicit/Partial & Partial & Yes \\ \hline
		\begin{tabular}[c]{@{}l@{}}Directional conflict\\ encoding\end{tabular} & Limited & \begin{tabular}[c]{@{}c@{}}Generally\\ Symmetric\end{tabular} & Yes \\ \hline
		\begin{tabular}[c]{@{}l@{}}Contextual priority\\ modulation\end{tabular} & Rare & None & \begin{tabular}[c]{@{}c@{}}Explicitly\\ available\end{tabular} \\ \hline
		\begin{tabular}[c]{@{}l@{}}Measurement / projection\\ separation\end{tabular} & None & None & Yes \\ \hline
		\begin{tabular}[c]{@{}l@{}}Operator family\\ approach\end{tabular} & None & Limited & Yes \\ \hline
		Scale regime separation & Limited & Limited & \begin{tabular}[c]{@{}c@{}}Explicitly\\ defined\end{tabular} \\ \hline
		\end{tabular}%
	}
	\caption{
		Conceptual comparison between existing conflict/divergence-based approaches and the proposed conflict framework in terms of structural properties.
		The table does not indicate performance superiority of methods; rather, it illustrates structural differences in how conflict is represented.
	}
	\label{tab:comparation-of-the-other-methods}
\end{table}

As shown in Table~\ref{tab:comparation-of-the-other-methods}, the proposed approach is not intended to replace existing methods.
Instead, it provides a more general level of abstraction in which the processes of conflict measurement, contextual modulation, and projection into the decision space are treated as structurally separated components.
This separation enables different conflict regimes to be modeled within a unified mathematical framework.

This approach is also conceptually related to the notions of disentangled representation and modular structures discussed in the modern representation learning literature
\cite{bengio2013representation}.
However, the objective of this study is not the disentanglement of learned representation spaces, but rather the structural definition of conflict measurement and projection processes under distinct operators.

\begin{remark}
In this study, conflict is not used in the sense of semantic contradiction or logical inconsistency; rather, it refers to a directional and contextually modulated structural discrepancy.
\end{remark}	


\section{Inputs and Representation Spaces}
\label{sec:inputs-and-representation-spaces}
	The proposed conflict model requires different types of inputs to be represented within explicitly defined
	representation spaces.
	
	In this section, the data matrices, priority structures, and the assumptions associated with these
	structures, on which the conflict operators are defined, are systematically presented in order to establish
	a consistent foundation for the operator family introduced in the subsequent sections.

\begin{definition}[Data Matrices]
\label{def:data-matrices}
	\( \epsilon > 0\) and \(
		\DataEpsilon = \left\{
			\boldsymbol{X} \in \Data \; \middle| \; x_{ij} \ge \epsilon, \; \forall{i, j}
		\right\}
	\) being given, the raw data and contextual data matrices are respectively defined as:
	\[
		\Raw, \Context \in \DataEpsilon
	\]

	Here, each row ($i \in \{ 1, \dots, r \}$) represents an instance, a decision unit, or a state;
	and each column ($j \in \{ 1, \dots, d \}$) represents a feature or criterion.

	This distinction enables conflict to be modeled in an instance-level yet feature-sensitive manner.
\end{definition}
\begin{remark}
	In this study, $\epsilon > 0$ does not represent a small-parameter assumption in the classical
	sense used in the analysis literature; rather, it is a \textbf{structural definition parameter}
	included in the definition of the data matrices that guarantees the
	\textbf{definitional and mathematical consistency} of the model.

	In this context, $\epsilon$ may be chosen arbitrarily small; however, it \textbf{cannot be equal to zero.}
	This is a necessary consequence of the principle of positivity and measurability
	(see Definition~\ref{def:positivity-and-measurability}) upon which the conflict framework is based.
\end{remark}
\begin{remark}[$\epsilon$ Assumption]
	In order for the conflict kernel to remain well-defined, all components are required to be positive.
	For this reason, the data matrices are defined on the set $\DataEpsilon$.
	Here, $\epsilon > 0$ is not a parameter introduced for numerical regularization purposes;
	rather, it is a structural constraint defining the domain of the model.

	Since zero-valued components may arise in practical applications, the data can be mapped into
	the set $\DataEpsilon$ through an appropriate positive transformation or regularization.
	The effect of such regularizations on conflict measurements depends on the application context
	and may additionally be examined in experimental analyses.
\end{remark}

\begin{assumption}[Computational Admissibility]
	The existence of a positive lower bound ($\epsilon > 0$) for the raw and contextual data matrices
	($\Raw, \Context$) is necessary in order to ensure that logarithmic ratios (log-odds) remain defined,
	that numerical stability \cite{higham2002accuracy} is preserved, and that well-definedness
	\cite{rudin1976principles, armstrong1988groups} is maintained.
	This assumption also provides a common domain for the entire conflict operator family.
\end{assumption}

\begin{definition}[Priority Structure and Matrix]
\label{def:priority-structual-matrix}
	Let each row $i = 1, \dots, r$ be a probability distribution.
	The interior simplex of the matrix space of dimension $r \times d$ is defined as:
	\[
		\RISimplex := \left\{
			\boldsymbol{X} \in \DataPos 
			\;\middle|\;
			\forall{i} \in \{ 1, \dots, r \},
			\sum_{j = 1}^{d} x_{ij} = 1
		\right\}
	\]

	From this, the priority/weight matrix (row-stochastic matrix),
	which determines the relative importance of features on a row-wise basis\footnote{
		The term ``priority matrix'' refers to the comparative ordering of rows;
		whereas the term ``weight matrix'' reflects the use of the same mathematical object
		in multi-criteria decision analysis.
		In this study, both terms are treated as synonymous.
	}, is defined as:
	\[
		\Priority \in \RISimplex
	\]
	
\end{definition}
\begin{remark}[Priority Normalization]
	In this study, each row of the priority matrix $\Priority$ is normalized such that it is defined
	over a probability simplex.

	This choice represents a modeling preference intended to express the relative priorities among
	criteria in an explicit and comparable manner.

	Nevertheless, the proposed conflict framework can, in principle, be extended to encompass
	alternative weighting regimes, such as non-normalized or sparse priority structures.
	Such generalizations may be investigated in future studies.
\end{remark}

\begin{definition}[Positivity and Measurability]
\label{def:positivity-and-measurability}
	According to this definition, for each instance $i$, ${\Priority}_{i, \cdot}$ is a normalized
	relative importance (priority) distribution defined over a probability simplex, and none of its
	components is zero (${\Priority}_{ij} > 0$).

	In this study, data is treated as a deterministic yet normalized hierarchy of relative importance
	that adopts the fundamental structural principles of normative utility in decision theory
	at the modeling level \cite{utility, luce}.

	Accordingly, the random-variable and sampling-based probabilistic framework commonly adopted
	in the classical statistical literature
	\cite{feller1968introduction, casella2002statistical}
	is not regarded as a necessary or foundational assumption of this study.

	Within the conflict framework, the raw and contextual data matrices ($\Raw, \Context$)
	and the priority/weight matrix ($\Priority$) satisfy the following condition:
	\[
		{\Raw}_{ij}\!>\!0, \;
		{\Context}_{ij}\!>\!0, \;
		{\Priority}_{ij}\!>\!0 \quad
		\forall{i}\!\in\!\{\!1,\!\dots,\!r \}, \;
		\forall{j}\!\in\!\{\!1,\!\dots,\!d \}
	\]

	The positivity constraint enables the distinction between the absence of conflict ($g = 0$)
	and directional dominance to be established in a mathematically consistent manner, and renders
	the domain of the conflict operators well-defined
	\cite{rudin1976principles, armstrong1988groups}.

	This positive definiteness is directly compatible with the fact that entropy functions and their
	derivatives in information theory are well-defined only over positive probability distributions
	\cite{entropy, cover-thomas}.

	In this context, the exclusion of zero-valued components guarantees that conflict analysis is
	performed within an ``information-rich'' representation space (i.e., a space in which all
	probability components are positive, ensuring that entropy functions and their derivatives remain
	well-defined), and provides a common and stable domain for the entire operator family.
\end{definition}


\section{Axiomatic Justification for Conflict Operators}
In this section, the defined conflict operator family
(see Section~\ref{sec:conflict-operator-family})
$\Act = \{ g_1, \; g_2, \; g_3 \}$ is constructed not arbitrarily, but upon an explicit axiomatic framework
concerning the \textbf{minimal mathematical principles that the notion of conflict must satisfy}.

The objective is to clearly establish which operators can meaningfully represent conflict within this framework.

In this context, the following axioms are regarded as necessary for treating conflict as a
\textbf{local, directional, and context-independent} mathematical object.

\begin{axiom}[Domain Compatibility and Admissibility]
\label{ax:conflict-operator-admissibility}
	An admissible conflict operator $g$ must be defined only over positive quantities
	(see Definition~\ref{def:positivity-and-measurability}):
	\[
		g : \left( \real_{> 0} \right)^{2} \to \real
	\]
\end{axiom}
\begin{rationale}
	The fact that the raw and contextual data matrices ($\Raw, \Context$) are defined over
	$\DataEpsilon$ (see Section~\ref{sec:inputs-and-representation-spaces}) and that the priority
	structure is restricted to the interior simplex \textbf{necessarily requires} the conflict kernel
	to operate within a representation space containing no zero-valued components.

	This guarantees the consistency of scale-invariant, normalized ratio-based ($g_1$), and logarithmic
	($g_2$) operators, while excluding domain-related problems such as division by zero and undefinedness.
\end{rationale}

\begin{axiom}[Zero-Conflict Consistency]
\label{ax:conflict-operator-zero-conflict-consistency}
	A conflict operator must \textbf{not produce conflict for identical inputs}:
	\[
		g( x, x ) = 0, \quad \forall{x} > 0
	\]
\end{axiom}
\begin{rationale}
	This enforces the absence of conflict (zero baseline) under equality; however, it does not require
	conflict to emerge whenever the inputs differ ($x \neq y$).

	Accordingly, conflict operators may produce the result $g( x, y ) = 0$ under sensitivity thresholds,
	contextual filters, or measurement tolerances.

	This distinction provides theoretical flexibility for directional dominance, signed conflict,
	and context-sensitive models.

	Otherwise, a structure in which identical values ($x = y$) could produce conflict would eliminate
	the reference point of the conflict notion itself, thereby violating definitional consistency.
\end{rationale}
\begin{remark}
	In this study, the absence of conflict is considered at two distinct levels.

	First, the \textbf{strict zero-conflict} condition arises when the raw and contextual values are
	identical, in which case the conflict measure is necessarily zero.

	However, in practical applications, certain tolerance or filtering mechanisms may be employed.
	In such cases, a \textbf{contextual zero-conflict} condition may arise; that is, values may be
	considered negligible or ignorable by the context even if they are not exactly equal.

	This distinction separates the mathematical properties of the conflict kernel from the decision
	mechanisms employed within the application context.
\end{remark}

\begin{axiom}[Antisymmetry]
\label{ax:conflict-operator-antisymmetry}
	Conflict is a directional quantity, and this directionality is modeled through the antisymmetry
	condition:
	\[
		g( x, y ) = -g( y, x ), \quad \forall{x, y} > 0
	\]

	Here, $x$ and $y$ represent non-directional quantitative inputs; the direction of conflict is
	produced solely by the conflict operator $g( \cdot, \cdot )$.
\end{axiom}
\begin{rationale}
	The relationship between raw and contextual data ($\Raw, \Context$) is not symmetric.

	This requires the conflict measure to encode not only magnitude, but also which component is dominant.

	Without such modeling, conflict reduces to symmetric (non-directional) distances, resulting in
	information loss in decision contexts.
\end{rationale}

\begin{axiom}[Local Continuity (Stability)]
\label{ax:conflict-operator-stability}
	Every admissible conflict operator is continuous over its domain:
	\[
		g \in C^{0} \left( \left( \real_{> 0} \right)^{2}, \real \right)
	\]
\end{axiom}
\begin{rationale}
	Small changes in the data should not produce uncontrolled discontinuities in conflict measurements.

	This axiom directly supports numerical stability, comparability, and integration with gradient-based
	methods.

	Although this axiom minimally requires $C^{0}$ continuity, the members of the proposed operator
	family $\Act$ ($g_1$, $g_2$, $g_3$) possess $C^{\infty}$ regularity (infinite differentiability)
	over their domains.

	This demonstrates that the framework is fully compatible with higher-order optimization methods.
\end{rationale}

\begin{axiom}[Explicit Scale Behavior]
\label{ax:conflict-operator-scale-behavior}
	The scale behavior of a conflict operator (scale-invariant, proportional, or absolute scale-sensitive)
	must be explicitly classifiable under one of these regimes according to its response to input scaling.
\end{axiom}
\begin{rationale}
	This guarantees that the response of the selected operator to unit changes in the data is known
	in advance.

	An operator with ambiguous scale behavior produces comparability problems across datasets or contexts
	with different magnitudes.
	For example, while the operator $g_1$ targets a dimensionless analysis, $g_3$ preserves absolute
	deviations.

	This explicitness is critical for the correct interpretation of conflict results within decision-making
	processes.
\end{rationale}

\subsection{Interpretation of the Fundamental Axioms for Conflict Operators}
\label{sub:basic-axioms-for-conflict-operators}
In this study, the axioms for conflict operators should be interpreted as reflections of the fundamental
mathematical principles inherent to the notion of conflict itself, independent of any specific operator.

In this context, \emph{these axioms are not operator-specific technical constraints}.
\begin{itemize}
	\item Axiom~\ref{ax:conflict-operator-admissibility}: Domain compatibility
	(positivity and definitional consistency),

	\item Axiom~\ref{ax:conflict-operator-zero-conflict-consistency}: Zero-conflict consistency
	(zero conflict for identical inputs),

	\item Axiom~\ref{ax:conflict-operator-antisymmetry}: Directionality and antisymmetry,

	\item Axiom~\ref{ax:conflict-operator-stability}: Local continuity and numerical stability,

	\item Axiom~\ref{ax:conflict-operator-scale-behavior}: Explicit scale behavior.
\end{itemize}

These principles are already satisfied by each element of the conflict operator family defined in
Section~\ref{sec:conflict-operator-family}, and the corresponding axioms are explicitly reflected
both in the operator definitions and in
Section~\ref{sec:conflict-operator-family-axiom-compatiblility}.

\begin{remark}[Necessity of the Joint Satisfaction of the Axioms]
The joint satisfaction of these axioms constitutes the minimal structural conditions that enable
a directional, context-sensitive, and scale-regime-separated representation of conflict.
\end{remark}

\subsection{Axiomatic Consequences: Fundamental Propositions}
\label{sub:axiomatic-consequences}
The propositions presented in this subsection do not introduce new axiomatic assumptions;
rather, under the given axioms, they formally characterize the structural behaviors expected
from conflict operators, optionally through additional operator properties such as
``discriminativity.''

The following propositions make explicit the fundamental structural behaviors that operators
within the conflict operator family $\Act$ are required to exhibit under the axioms.

\begin{definition}[Discriminative Conflict Operators]
\label{def:discriminative-conflict-operators}
	A conflict operator $g$ is called discriminative if
	\[
		g( x, y ) = 0 \quad \implies \quad x = y, \qquad \forall{x, y} > 0
	\]
\end{definition}

\begin{proposition}[Characterization of Conflict Absence by Identity]
\label{prop:conflict-free-consistency}
	For every discriminative operator $g \in \Act$ within the conflict operator family $\Act$
	defined under
	Axiom~\ref{ax:conflict-operator-admissibility}
	and
	Axiom~\ref{ax:conflict-operator-zero-conflict-consistency},
	the following holds:
	\[
		g( x, y ) = 0 \quad \iff \quad x = y
	\]
\end{proposition}
\begin{rationale}
	Axiom~\ref{ax:conflict-operator-zero-conflict-consistency} guarantees that identical inputs
	produce zero conflict.

	In addition, for the notion of conflict to be interpreted as a discriminative and directional
	quantity, zero conflict must occur only for identical inputs.

	Otherwise, the production of zero conflict between non-identical inputs ($x \neq y$) would
	eliminate the discriminative capability of the conflict measure, rendering directional comparison
	and dominance interpretations meaningless.

	This proposition explicitly clarifies under which conditions the absence of conflict should be
	accepted for the conflict operator family considered in this study.
\end{rationale}

\begin{proposition}[Directional Consistency]
\label{prop:conflict-directional-consistency}
	Under Axiom~\ref{ax:conflict-operator-antisymmetry}, the conflict operator is directionally
	consistent; that is:
	\[
		g( x, y ) > 0 \quad \iff \quad g( y, x ) < 0
	\]
\end{proposition}
\begin{rationale}
	The axiom \nameref*{ax:conflict-operator-antisymmetry} requires conflict to encode not only
	magnitude, but also direction.

	This structure enables dominance and opposition relationships to be distinguished directionally.
	Without antisymmetry, conflict measurement reduces to a non-directional difference function.
\end{rationale}

\begin{proposition}[Local Stability]
\label{prop:conflict-local-stability}
	Under Axiom~\ref{ax:conflict-operator-stability}, the conflict operator
	$g : ( \real_{> 0} )^{2} \to \real$ is continuous at every point of its domain.

	In other words, sufficiently small perturbations in the input space produce arbitrarily small
	changes in the conflict value:
	\[
		\forall{( x_0, y_0 )} \in ( \real_{> 0} )^{2},
		\forall{\varepsilon > 0}, \; \exists \delta > 0
	\]
	such that
	\[
		\| ( x, y ) - ( x_0, y_0 ) \| < \delta
		\Rightarrow | g( x, y ) - g( x_0, y_0 ) | < \varepsilon
	\]

	Accordingly, the conflict kernel $\Phi^{(g)}$ defines a locally stable transformation.

	By the component-wise definition of the conflict kernel $\Phi^{(g)}$
	(Equation~\ref{eq:conflict-kernel-definition}),
	$\Phi^{(g)}$ is obtained through finitely many applications and algebraic combinations
	of the continuous function $g$.

	Since finite compositions and combinations of continuous functions preserve continuity,
	the continuity of $g$ guarantees the continuity of $\Phi^{(g)}$.
\end{proposition}
\begin{rationale}
	This result guarantees not only that the conflict operator is well-defined, but also that it
	produces stable representations under measurement errors, numerical noise, and data perturbations.

	Accordingly, conflict measurement can be used as a reliable quantity in decision-support and
	optimization contexts.
\end{rationale}
	

 	\section{Conflict Operator Family}
 	\label{sec:conflict-operator-family}
	In this section, the operator structure determining the quantitative properties of conflict is defined,
	and the conflict operator family consisting of admissible conflict operators is introduced axiomatically.

 	This definition is kept sufficiently general to allow operators that can model conflict as symmetric
 	or antisymmetric, and whose output range may be bounded or unbounded, provided that they satisfy
 	the admissibility axioms defined in this study.	
 	
 	The conflict operators introduced in this section are based on fundamental mathematical transformations
 	that have appeared in different contexts in the literature.
 	For example, transformations such as the normalized difference (see Eq.~\ref{eq:canonical-scale-invariant-conflict}),
 	the log-ratio (see Eq.~\ref{eq:log-odds-conflict}), and the simple difference (see Eq.~\ref{eq:raw-difference-conflict})
 	are widely used in various fields such as measurement theory, odds-ratio analysis, and comparative
 	evaluation problems.
	 	
 	However, the aim of this study is not to propose a new divergence or metric.
	Instead, it defines \textbf{an axiomatic conflict operator family} within which such transformations,
	which arise in different contexts in the literature, can be situated.
 	
 	The proposed framework does not impose a specific conflict function. This approach makes it possible
 	to treat different conflict kernels within a common axiomatic framework. Thus, processes such as conflict
 	measurement, contextual priority modulation, and decision projection are separated from one another,
 	yielding a modular analysis structure.
 	
 	\begin{definition}[Admissible Conflict Operator]
	 	A conflict operator $g$ is defined as a function that transforms the discrepancy between two
	 	positive scalar quantities into a directional and quantitative measure:
	 	\[
	 		g : \left( \real_{> 0} \right)^{2} \to \real
	 	\]
 	\end{definition}
 	
 	This definition is deliberately kept broad; additional regularity or symmetry conditions may be imposed
 	in specific applications.
	
	\begin{definition}[Operator Family]
	\label{def:conflict-operator-family}
		Here, three fundamental operators are defined in order to examine conflict measures with different
		scale, boundedness, and symmetry properties under a single framework:
		\[
			\Act = \{ g_1, \; g_2, \; g_3 \}
		\]
	\end{definition}
	\begin{enumerate}
		\item \textbf{Canonical Scale-Invariant Conflict ($g_1$):}
		\begin{equation}
		\label{eq:canonical-scale-invariant-conflict}
			g_1( x, y ) = \frac{x - y}{x + y} \in ( -1, +1 )
		\end{equation}
		
		This operator is a \textbf{canonical conflict measure} because it expresses the magnitude of conflict
		within a normalized interval and remains invariant under positive scale transformations.
		
		The boundedness of the output within the interval $(-1, +1)$ enables comparability across different
		rows and contexts.
		
		The operator has the property of \textit{antisymmetry}:
		\[
			g_1( x, y ) = -g_1( y, x )
		\]
		which allows conflict to be interpreted as a directional (signed) quantity.
		
		The operator $g_1( x, y )$ has the following properties:
		\begin{itemize}
			\item Scale invariance,
			\item Boundedness: $g_1( x, y ) \in ( -1, +1 )$,
			\item Continuity and differentiability (for $x, y > 0$),
			\item Antisymmetric structure.
		\end{itemize}
		
		These properties make it possible for the conflict score to be used stably together with
		priority weights ($\Priority$), both as a normalized and directional quantity.
	
		\item \textbf{Logarithmic Ratio (Log-odds, $g_2$):}
			\begin{equation}
			\label{eq:log-odds-conflict}
				g_2( x, y ) = \ln{\left( \frac{x}{y} \right)} \in \real
			\end{equation}
		
			This operator ($\ln$, the natural logarithm) expresses the proportional difference between two
			positive quantities in logarithmic space.
			
			The unboundedness of the output enables high-amplitude or exponential-scale conflicts to be
			represented without being suppressed.
			
			For this reason, $g_2$ focuses not on absolute differences, but on relative dominance relations,
			and in this respect it is directly compatible with information-theoretic and odds-ratio
			(log-odds)-based analyses.
			
			The operator has the property of \textit{antisymmetry}:
			\[
				g_2( x, y ) = -g_2( y, x )
			\]
			which enables the direction of conflict (which component is dominant) to be explicitly encoded.
	
			The operator $g_2( x, y )$ has the following properties:
			\begin{itemize}
				\item Scale invariance,
				\item Ratio-sensitivity,
				\item Antisymmetric structure,
				\item Continuity and differentiability (for $x, y > 0$),
				\item Unbounded output space: $g_2( x, y ) \in \real$.
			\end{itemize}
			
			These properties make $g_2$ a suitable choice in high-contrast or information-theoretic contexts
			where the direction of conflict is more decisive than its magnitude.
			
		\item \textbf{Raw Difference:}
			\begin{equation}
			\label{eq:raw-difference-conflict}
				g_3( x, y ) = x - y \in \real
			\end{equation}
			This operator measures the conflict between two quantities directly through the \textit{raw difference}.
			
			Since it does not involve scaling or normalization, it is used as a reference (baseline) conflict
			measure in cases where the preservation of quantities with physical meaning (energy, cost,
			density, etc.) is important.
			
			The operator has the property of \textit{antisymmetry}:
			\[
				g_3( x, y ) = -g_3( y, x )
			\]
			however, unlike the other operators, it directly depends on the absolute scale of the quantities.	
			
			The operator $g_3( x, y )$ has the following properties:
			\begin{itemize}
				\item Linear structure,
				\item Antisymmetric structure,
				\item Continuity and differentiability,
				\item Scale dependence.
			\end{itemize}
	\end{enumerate}
	\begin{remark}[Interpretation of Operator Selection]
		The defined conflict operator family $\Act = \{ g_1, \; g_2, \; g_3 \}$ has been selected to cover
		fundamental conflict regimes ranging from scale-independent normalized measures to proportional and
		information-theoretic expressions, as well as absolute-scale linear differences.
		
		This selection enables the framework to be adapted to different application contexts without being
		tied to a specific measure.
	\end{remark}
	
	\subsection{Fundamental Mathematical Properties of the Conflict Family}

	\begin{proposition}[Well-Definedness of Conflict Operators]
	\label{prop:well-defined}
		Let $\Act = \{ g_1, \; g_2, \; g_3 \}$. Each operator $g \in \Act$ is a single-valued
		and deterministic function:
		\[
			g : \left( \real_{> 0} \right)^{2} \to \real
		\]

		$\RSpace^{(g)}$ represents the image set of the operator; the transformations $\Phi^{(g)}$
		and $\Gamma^{(g)}$ defined in the subsequent sections are constructed over this space.
		
		Therefore, the structures $\Phi^{(g)}$ and $\Gamma^{(g)}$ hierarchically constructed over the
		operator family $\Act$ are well-defined over the entire domain and produce unique outputs for
		identical inputs.
	\end{proposition}
	\begin{proof}
		Let $x, y \in \real_{> 0}$:
		\begin{enumerate}[(i)]
			\item \textbf{Canonical Scale-Invariant Conflict ($g_1$):}
			\[
				g_1( x, y ) = \frac{x - y}{x + y}
			\]
			
			Since the denominator satisfies $x + y > 0$, $g_1$ is defined over the entire domain.
			
			Moreover,
			\[
				| x - y | < x + y \quad \Rightarrow \quad g_1( x, y ) \in ( -1, +1 ), 
			\]
			and therefore the outputs of $g_1$ are compatible with $\RSpace^{(g_1)}$.
			
			\item \textbf{Logarithmic Ratio (Log-odds, $g_2$):}
			Since $x / y > 0$ for $x, y > 0$,
			\[
				g_2( x, y ) = \ln \left( \frac{x}{y} \right)
			\]
			is well-defined over the entire domain and its outputs lie in $\RSpace^{(g_2)}$.
			
			\item \textbf{Raw Difference ($g_3$):}
			\[
				g_3( x, y ) = x - y
			\]
			its outputs are compatible with $\RSpace^{(g_3)}$ since subtraction is closed over the real
			numbers.
		\end{enumerate}
		
		Thus, each member of the operator family $\Act$ is compatible with the predefined output spaces
		and ensures definitional stability for the conflict kernel $\Phi^{(g)}$ and the framework
		$\Gamma^{(g)}$.
	\end{proof} 
	\begin{remark}[Interpretation and Framework-Level Meaning]
		This proposition ensures not only that the conflict operators are functionally defined, but also
		that the output spaces remain stable throughout the kernel-framework construction.		
		
		Thus, conflict measurements can be hierarchically extended without introducing definitional ambiguity.
	\end{remark}
	
	\begin{proposition}[Boundedness of Conflict Operators]
	\label{prop:boundedness}
		The canonical scale-invariant conflict operator $g_1$,
		\[
			g_1( x, y ) = \frac{x - y}{x + y}, \quad x, y \in \real_{> 0}
		\]
		is strictly bounded under its definition:
		\[
			g_1( x, y ) \in ( -1, +1 )
		\]
		
		In contrast, the output spaces of the logarithmic ratio operator $g_2$ and the raw difference
		operator $g_3$ are unbounded over $\real$.
	\end{proposition}
	\begin{proof}
		Let $x, y \in \real_{> 0}$. Then,
		\[
			| x - y | < x + y
		\]
		always holds. 
		
		Therefore,
		\[
			-1 < \frac{x - y}{x + y} < 1
		\]
		and hence $g_1$ is strictly bounded.		
		
		On the other hand, since the ratio $x / y$ in the expression
		$g_2( x, y ) = \ln \left( \dfrac{x}{y} \right)$ can approach $0$ or $+\infty$,
		the outputs of $g_2$ are unbounded over $\real$.		
		
		Similarly, for the expression $g_3( x, y ) = x - y$, since $x$ and $y$ can grow independently,
		the output set is unbounded.
	\end{proof}
	\begin{remark}[Distinction Between Normalized and Absolute Conflict]
		This proposition shows that the conflict operator family brings together normalized
		(scale-invariant) and absolute (scale-sensitive) conflict measurements under the same framework.		
		
		This distinction enables transitions between the contextual interpretation of conflict and
		absolute magnitude analysis.
	\end{remark}
	
	\begin{proposition}[Weighted Conflict Conservation]
	\label{prop:conservation}
		Let $\Priority \in \RISimplex$. Under the weighted conflict structure defined by the Hadamard
		product ($\odot$),
		\[
			\Priority \; \odot \; \Phi^{(g)}( \Raw, \Context )
		\]
		each local conflict component $( i, j )$ is scaled only by its corresponding priority weight
		$\Priority_{ij}$, and no inter-component interference occurs.
		
		Therefore, the local conflict representation prior to aggregation is structurally preserved.
	\end{proposition}
	\begin{proof}
		Since the Hadamard product is defined component-wise (element-wise), for each $( i, j )$,
		\[
			\left( 
				\Priority \odot \Phi^{(g)}( \Raw, \Context ) 
			\right)_{ij} = \Priority_{ij} \cdot \Phi^{(g)}( \Raw, \Context )_{ij}
		\]
		which shows that no component $( k, \ell ) \neq ( i, j )$ contributes to the output at position
		$( i, j )$.
		
		Thus, the weighting operation creates only a local scaling effect.
		By the definition $\Priority \in \RISimplex$, for each row $i$,
		$\sum_{j = 1}^{d} \Priority_{ij} = 1$ and $\Priority_{ij} > 0$ hold; however, this normalization
		does not disrupt the independence between different $( i, j )$ positions.
		
		Consequently, the weighting operation causes neither global redistribution nor structural distortion;
		local conflict information is preserved by being modulated only by its own priority weight.
	\end{proof}
	

	\section{Axiomatic Compatibility of the Current Operator Family}
	\label{sec:conflict-operator-family-axiom-compatiblility}
	The defined operator family
	\[
		\Act = \{ g_1, \; g_2, \; g_3 \}
	\]
	satisfies \textbf{all} the axioms given in
	Section~\ref{sec:conflict-operator-family} and Section~\ref{sub:basic-axioms-for-conflict-operators}:
	\begin{itemize}
		\item $g_1$: Scale-invariant, bounded, antisymmetric $\to$ 
		\textbf{normalized comparative conflict}.
		
		\item $g_2$: Ratio-sensitive, scale-invariant, unbounded, antisymmetric $\to$
		\textbf{information-theoretic / log-odds regime}.
		
		\item $g_3$: Linear, absolute-scale, antisymmetric $\to$
		\textbf{physical or quantitative reference regime}.
	\end{itemize}
	(See Axiom~\ref{ax:conflict-operator-admissibility}, 
	\ref{ax:conflict-operator-zero-conflict-consistency}, \ref{ax:conflict-operator-antisymmetry},
	\ref{ax:conflict-operator-stability}, \ref{ax:conflict-operator-scale-behavior}.)	
	This triplet constitutes a \textbf{minimal yet representative} set of operators satisfying all axioms. 
	
	Other operators may be added to this family; however, transformations that violate any of the axioms
	cannot be regarded as ``conflict operators.''


	\section{Conflict Kernel}
	\label{sec:conflict-kernel}
	In this section, it is defined how the conflict between the raw data ($\Raw$) and contextual data
	($\Context$) matrices is generated systematically and in a structured manner through a selected
	conflict operator.
	
	For this purpose, a \textit{conflict kernel}, which computes conflict in an instance-level and
	feature-sensitive manner, is introduced.
	
	\begin{definition}[Output Space Induced by the Conflict Operator]
	Let an admissible conflict operator $g \in \Act$ be fixed.
	
	The output space ${\RSpace}^{(g)}$ induced by this operator is defined as the set of all possible
	values that $g$ can take:
	\begin{equation}
	\label{eq:output-space-for-g}
		{\RSpace}^{(g)} := \begin{cases}
			( -1, +1 )^{r \times d} & g = g_1, \\
			\Data & g = g_2, g_3
		\end{cases}
	\end{equation}
	
	This definition creates a distinct output geometry depending on whether the conflict operator produces
	outputs in a scaled or raw data space.
	\end{definition}
	
	\begin{definition}[Conflict Kernel]
	\label{def:conflict-kernel-operator}
		For a selected conflict operator $g \in \Act$, the conflict kernel is a transformation defined as:
		\begin{equation}
		\label{eq:conflict-kernel-formal}
			\Phi^{(g)} : \left( \DataEpsilon \right)^{2} \to \RSpace^{(g)}
		\end{equation}
		
		Here, $\RSpace^{(g)}$ denotes the representation space defined according to the output character
		of the selected operator.
		
		The kernel operates at each position $(i, j)$ as follows:
		\begin{equation}
		\label{eq:conflict-kernel-definition}
			\Phi^{(g)}( \Raw, \Context )_{ij} \; := \; g\!\left( {\Raw}_{ij}, {\Context}_{ij} \right)
		\end{equation}		 
		
		This transformation encodes the local conflict between the raw data ($\Raw$) and contextual data
		($\Context$) in accordance with the structure of the selected operator.
	\end{definition}	
		

	\section{Fusion Operator}
	\label{sec:fusion-operator}
	Although the conflict kernel ($\Phi^{(g)}$) reveals in detail the local and operator-dependent
	conflict structure between decision dimensions and criteria, this structure does not directly produce
	a global decision signal.
	
	Therefore, conflict information must be reduced to a single or low-dimensional representation, or
	projected into specific decision contexts. This requirement makes it necessary to define a
	\textbf{fusion operator} that enables the conflict matrix to be treated holistically.
	
	The fusion operator is treated as an abstract transformation that brings together the multidimensional
	conflict information produced by ($\Phi^{(g)}$) in a manner compatible with the priority structure
	and the application context.
	
	This operator may or may not be linear; it is not required to be reduced to a fixed functional form.
	On the contrary, in the proposed framework, the fusion step is deliberately kept flexible in order to
	provide an adaptable structure for different decision strategies, projection geometries, and application
	scenarios.
	
	\begin{definition}[Fusion Operator]
	\label{def:fusion-operator}
		The fusion operator is a transformation that maps the output ($\Output$) in the conflict space
		generated depending on the selected conflict operator ($g$) to a higher-level representation or
		decision space:
		\[
			\Psi : {\RSpace}^{(g)} \to \Output,
			\quad \text{where} \quad
			\Output \in \left\{ \Data, \real^{r}, \Simplex \right\}
		\]
	
		This operator may preserve, reduce, normalize, or combine local conflict values under a probabilistic
		structure.
	\end{definition}
	\begin{remark}[Minimum Conditions for the Fusion Operator]
		The operator $\Psi$ is a mechanism that enables the conflict matrix to be interpreted according
		to the application context or transformed into a higher-level representation space. 
	
		The form of this operator may vary depending on the problem domain; however, in practical
		applications, it is generally expected to satisfy basic well-posedness properties such as continuity,
		stable behavior with respect to changes in the inputs, and the production of bounded outputs from
		bounded inputs.
	\end{remark}
	

	\section{Conflict Framework}
	\label{sec:conflict-framework-generalized}
	When the data representations, priority structure
	(see Section~\ref{sec:inputs-and-representation-spaces}),
	conflict operator family (see Section~\ref{sec:conflict-operator-family}), and fusion mechanism
	(see Section~\ref{sec:fusion-operator}) defined in the previous sections are considered together,
	they enable the construction of a closed and generalizable conflict model. 	
	
	In this section, a \textbf{generalized} and \textbf{application-independent} conflict framework obtained
	through the composition of these components is introduced. The proposed framework formulates conflict
	as an operator-based and contextual transformation process modulated by priority weights.
	This approach makes it possible to represent both the local (component-level) and global
	(decision-level) effects of the notion of conflict within the same mathematical structure, while ensuring
	that the framework remains compatible with practical requirements such as differentiability,
	computational complexity, and extensibility.
	
	The idea of contextual modulation has gained an increasingly central role in modern artificial
	intelligence systems through attention and adaptive weighting mechanisms \cite{vaswani2017attention}.
	However, in the proposed framework, the modulation process is carried out through explicitly defined
	priority operators rather than learned attention distributions.
	
	\begin{definition}[Conflict Framework]
		For a selected member $g \in \Act$ of the admissible conflict operator family, the
		\textit{generalized conflict framework} is defined as
		\[
			{\Gamma}^{(g)} : \left( \DataEpsilon \right)^{2} \times \RISimplex \; \to \; \Output
		\]
		
		Given raw and contextual data matrices $\Raw, \Context \in \DataEpsilon$ and a priority matrix
		$\Priority \in \RISimplex$, the framework $\Gamma^{(g)}$ is defined by the following operator
		composition:
		\[
			\Gamma^{(g)}_{\Phi, \Psi}( \Raw, \Context, \Priority ) \; := \;
			\Psi \! \left(
				\Priority \odot \Phi^{(g)}( \Raw, \Context )
			\right)
		\]
		
		Here:
		\begin{itemize}
			\item \textbf{$\Phi^{(g)}$ (Internal Conflict Transformation):}
			A deterministic transformation that maps the local discrepancy between raw data and context
			into a ``conflict kernel space'' determined by the structural properties of the selected conflict
			operator $g$ (scale invariance, antisymmetry, etc.).
			
			\item \textbf{$\Priority \odot \Phi^{(g)}$ (Hadamard Modulation):}
			Enables the computed local conflicts to be weighted at the feature level according to the priority
			regime of the decision maker or the system. This step incorporates not only the presence of
			conflict, but also its contextual importance into the model.
			
			\item \textbf{$\Psi$ (External Projection):} 
			Maps the modulated conflict information to an output space appropriate to the nature of the
			problem (score, probability distribution, or gradient correction).
		\end{itemize}
		
		This structure enables different conflict operators to be treated within a single unified formulation
		under different priority regimes and projection mechanisms.
	\end{definition}
	
	For the proposed operator $\Gamma^{(g)}$, if the total number of elements is taken as $N = r \times d$,
	the computational complexity will be of order $O( N )$.	This linear complexity enables the framework to
	operate stably and scalably on large-scale datasets and high-dimensional representation spaces.
	
	The differentiability of the selected operators $g \in \{ g_1, g_2, g_3 \}$ and of commonly preferred
	$\Psi$ projections (e.g., linear summation, softmax) allows the framework $\Gamma^{(g)}$ to be directly
	integrated into end-to-end learning architectures and gradient-based optimization processes.

	\begin{proposition}[Framework-Level Well-Definedness of the Conflict Framework]
		Well-definedness means that every admissible input is
		\begin{enumerate}[(i)]
			\item processed without leaving the domain and
			\item mapped to a unique and determinate output.
		\end{enumerate}
		Accordingly, for any selected admissible conflict operator
		$g \in \Act = \{ g_1, \; g_2, \; g_3 \}$ and any well-defined fusion operator
		$\Psi$,
		\[
			\Gamma^{(g)}_{\Phi, \Psi} : \left( \DataEpsilon \right)^{2}
				\times \RISimplex
				\; \to \; \Output
		\]
		the generalized conflict framework defined in this form is well-defined.
		
		Hence, as long as the conflict operator $g$ is well-defined,		
		\begin{enumerate}
			\item the conflict kernel $\Phi^{(g)}$,
			\item the Hadamard modulation $\Priority \odot \Phi^{(g)}$,
			\item and the conflict framework $\Gamma^{(g)}_{\Phi, \Psi}$ defined by their composition
		\end{enumerate}
		are well-defined over their entire domains.
	\end{proposition}
	\begin{proof}
		The components constituting the conflict framework are well-defined sequentially.
		\begin{enumerate}[(i)]
			\item \textbf{Well-definedness of conflict operators:}
			Each operator $g \in \Act$ is defined as
			$g : \left( \real_{> 0} \right)^{2} \to \real$ and produces a unique output for all inputs
			$x, y > 0$
			(see Eq.~\ref{eq:canonical-scale-invariant-conflict}, \ref{eq:log-odds-conflict},
			\ref{eq:raw-difference-conflict}). Therefore, the conflict operator $g$ is well-defined.
			
			\item \textbf{Well-definedness of the conflict kernel:}
			Under the assumption $\Raw, \Context \in \DataEpsilon$ and the constraint $\epsilon > 0$,
			we have ${\Raw}_{ij} > 0, \quad {\Context}_{ij} > 0$ for every $(i, j)$.
			Therefore, the expression
			$\Phi^{(g)}( \Raw, \Context )_{ij} = g( {\Raw}_{ij}, {\Context}_{ij} )$
			is well-defined for each component and lies in the space $\RSpace^{(g)}$ depending on the
			selected operator.
			
			\item \textbf{Well-definedness of priority modulation:}
			Since the priority matrix $\Priority \in \RISimplex$, all its components are positive and finite:
			${\Priority}_{ij} > 0, \quad \forall{i, j}$.
			Since the Hadamard product ($\odot$) is an element-wise operation,
			the expression $\Priority \odot \Phi^{(g)}( \Raw, \Context )$ is well-defined for each $(i, j)$.
			
			\item \textbf{Composition with the fusion operator:}
			Under the assumption that the transformation $\Psi : \RSpace^{(g)} \to \Output$ is well-defined,
			the composition of functions is also well-defined.
			Hence,
			\[
				\Gamma^{(g)}_{\Phi, \Psi}( \Raw, \Context, \Priority ) = \Psi\!\left(
					\Priority \odot \Phi^{(g)}( \Raw, \Context )
				\right)
			\]
			produces a unique output for every input $( \Raw, \Context, \Priority )$.
		\end{enumerate}		
	\end{proof}
	
	\begin{remark}[Structural Decomposition Principle]
	\label{rem:structural-decomposition-principle}
		In order for conflict to be evaluated as directional, contextual, and priority-sensitive,
		the components of raw and contextual data ($\Raw, \Context$) and relative importance ($\Priority$)
		must be defined as structurally separate components rather than being combined under a single tensor
		or score.

		Conflict measurements performed under a singular representation entangle scale behavior,
		directional information, and contextual effect in a non-separable manner.		
		This makes it impossible to interpret the source of conflict and eliminates comparability across
		different application regimes. The proposed decomposition enables component-based and
		operator-sensitive analysis of conflict.
	\end{remark}
	
	\begin{remark}[Positivity and Well-Definedness Link]
	\label{rem:positivity-well-definedness-link}
		The positive definiteness of all data components within the conflict framework is necessary for
		the domain of conflict operators to remain well-defined.
		
		The presence of zero or negative components may cause ratio-based and information-based conflict
		operators to collapse definitionally or to produce meaningless boundary cases.
		
		The positivity constraint makes the distinction between the absence of conflict and directional
		dominance mathematically consistent and guarantees that the operator family produces values over a
		common representation space.
	\end{remark}
	
	\begin{remark}[Operator Family Approach]
		A single conflict measure cannot universally represent conflict behavior across different scale
		regimes and contextual scenarios. Therefore, conflict analysis should be carried out through an
		operator family constrained by axioms.
		
		Different application domains are sensitive to the absolute magnitude, relative ratio, or
		scale-independent structure of conflict. A single measure cannot simultaneously satisfy all these
		requirements. The operator family approach allows these different regimes to be modeled consciously
		and controllably.
	\end{remark}
	
	\begin{remark}[Measurement--Reduction Separation]
	\label{rem:separation-measurement-aggregation}
		The measurement of conflict and its reduction to a decision score must be kept conceptually separate.
		
		A conflict measure is a multidimensional object carrying directional, scale, and contextual
		information. Reducing this information to a single score at an early stage leads to irreversible
		information loss in the decision-making process.
		
		A separate reduction operator ($\Psi$) allows application-dependent decision strategies to be
		defined independently of conflict measurement.
	\end{remark}

	\begin{proposition}[Kernel Well-Definedness]
	\label{prop:kernel-well-definedness}
		Every conflict operator $g \in \Act$ satisfying the positivity and domain compatibility axioms
		is well-defined over the data representation space $\DataEpsilon$ and produces a finite,
		real-valued, unique output.
	\end{proposition}
	\begin{proof}
		The positivity constraint excludes domain problems such as division by zero, logarithmic undefinedness,
		and sign ambiguity.
		
		The domain compatibility axiom requires conflict operators to produce outputs only over
		$\real_{> 0}$. Under these conditions, every $g \in \Act$ defined over $\DataEpsilon$ constitutes
		a mathematically consistent and well-defined conflict kernel.
	\end{proof}
	
	\begin{proposition}[Scale Regime Separation]
	\label{prop:scale-regime-separation}
		The operators in the family $\Act = \{ g_1, \; g_2, \; g_3 \}$ represent distinct regimes in terms
		of their behavior under positive scalar scaling $T_c( x ) = cx$ ($c > 0$):
		$g_1$ and $g_2$ are scale-invariant, whereas $g_3$ is scale-sensitive.
		
		Therefore, these operators cannot be interpreted as reducible to a single operator within an
		equivalent equivariant/invariant behavior class under scaling.
	\end{proposition}
	\begin{proof}
		For $c > 0$,
		$g_1( cx, cy ) = \frac{cx - cy}{cx + cy} = g_1( x, y )$ and
		$g_2( cx, cy ) = \ln\!\left( \frac{cx}{cy } \right) = \ln\!\left( \frac{x}{y} \right) = g_2( x, y )$
		are obtained.
		
		On the other hand, $g_3( cx, cy ) = cx - cy = c( x - y) = c\,g_3( x, y )$, and hence it is not
		scale-invariant.
	\end{proof}
	
	\begin{remark}[Role of the Axioms and Relaxation Discussion]
	\label{rem:role-of-axioms}
		The fundamental conflict axioms in this study (Axiom~\ref{ax:conflict-operator-admissibility}--
		\ref{ax:conflict-operator-scale-behavior}) provide the minimal structural guarantees required for
		conflict to remain a directional, well-defined, and comparable quantity.
		
		Relaxing any of these axioms may either expose conflict measurement to definitional problems
		(division by zero, logarithmic undefinedness, etc.) or make the interpretation of direction/scale
		ambiguous; therefore, such relaxations should be justified only through application-dependent
		additional assumptions.
	\end{remark}
		

	\section{Experimental Validation 1: Numerical Verification of Axiomatic Behaviors}
	
	\subsection{Objective of the Experiment}
	The objective of this experiment is to numerically verify the axiomatic behaviors and geometric
	properties of the proposed conflict operator family.
	
	In particular, the following properties are intended to be observed experimentally:	

	\begin{itemize}
		
		\item \textbf{Zero-Conflict Consistency:} 
		$g( x, x ) = 0$
		
		\item \textbf{Antisymmetry:} 
		$g( x, y ) = -g( y, x )$
		
		\item \textbf{Local Stability:} 
		Small changes in the input space do not produce discontinuous jumps in the output.
		
	\end{itemize}

	In addition, this experiment also aims to visualize that the selected operators generate
	\textbf{conflict geometries}.
	
	\subsection{Experimental Setup}
	The experiments were performed over the positive input space:
	\[
		( x, y ) \in [ a, b ]^{2}, \quad 0 < a < b
	\]
	For visualization purposes, the following interval was selected:
	\[
		a = 0.1, \quad b = 10
	\]
	This interval covers both small and large positive values, allowing the behavior of the operators
	to be observed with sufficient diversity.
	
	A \textbf{200x200 grid} was constructed over the experimental space, and each operator was evaluated
	on this grid.
	
	The operators examined are:
	\[
		g_1( x, y ) = \frac{x - y}{x + y}
	\]
	\[
		g_2( x, y ) = \ln \left( \frac{x}{y} \right)
	\]
	\[
		g_3( x, y ) = x - y
	\]
	
	The behavior of the operators was analyzed using both numerical metrics and contour/heat maps.
	
	\subsection{Numerical Verification Results}
	The experimental results are shown in tabular form below:
	
	\begin{table}[!ht]
		\resizebox{\columnwidth}{!}{%
			\begin{tabular}{|l|c|c|l|}
				\hline
				\rowcolor[HTML]{C0C0C0} 
				\textbf{Op.} & \textbf{max|g(x,x)} & \textbf{max|g(x,y)+g(y,x)} & \textbf{Behavior} \\ \hline
				$g_1$ & $\sim$0 & $\sim$0 & \begin{tabular}[c]{@{}l@{}}bounded, \\ normalized\end{tabular} \\ \hline
				$g_2$ & $\sim$0 & $\sim$0 & ratio-sensitive \\ \hline
				$g_3$ & $\sim$0 & $\sim$0 & linear \\ \hline
			\end{tabular}%
		}
		\caption{
			Numerical verification results of the proposed conflict operator family in terms of axiomatic
			properties.
			The table shows that the zero-conflict consistency ($g(x,x)=0$) and antisymmetry
			($g(x,y)=-g(y,x)$) conditions are satisfied at machine precision level for all operators, and
			that the operators produce different conflict behavior regimes.
		}
	\end{table}

	The results show that the \textbf{zero-conflict consistency} and \textbf{antisymmetry} properties are
	satisfied at machine precision level for all operators.
	
	This finding verifies that the proposed operators are consistent with the axiomatic framework defined
	in the paper.
		
	\subsection{Visualization of Operator Geometry}
	The contour maps of the three operators are respectively presented below.
	These visualizations clearly reveal the following geometric differences:
	
	\subsubsection*{\textbf{$g_1$ Operator}}
	\[
		g_1( x, y ) = \frac{x - y}{x + y}
	\]
	This operator produces a \textbf{normalized and bounded conflict surface}.
	The output range is bounded and relatively stable against scale changes.
	
	This property provides an advantage in cases where comparisons must be made across different magnitude
	scales.
	
	\par\medskip\noindent
	\begin{minipage}{\linewidth}
  		\centering
 		\resizebox{0.9\linewidth}{!}{\includegraphics{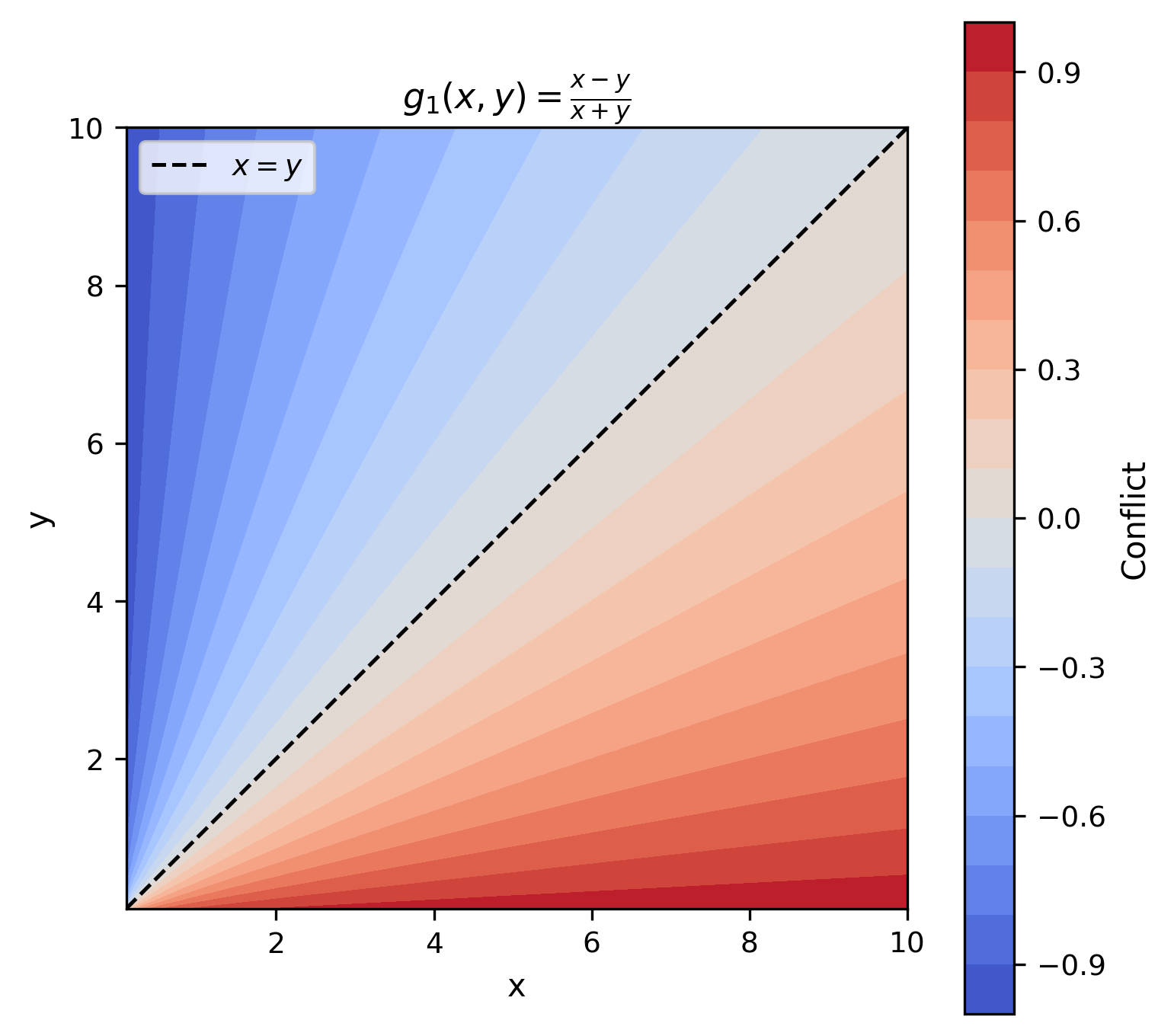}}
  		\captionof{figure}{
    		Contour/heat map of the operator $g_1( x, y ) = \frac{x - y}{x + y}$ over the positive input
    		space. The diagonal line $x = y$ indicates zero conflict.
    		The operator produces a bounded and normalized conflict surface.
  		}
  		\label{fig:experiment-1-operators-visualization-heatmap-g1}
	\end{minipage}
	\par\medskip 
	
	\subsubsection*{\textbf{$g_2$ Operator}}
	\[
		g_2( x, y ) = \ln \left( \frac{x}{y} \right)
	\]
	This operator measures conflict in a \textbf{ratio-based} manner.
	It is sensitive to magnitude ratios rather than absolute differences.
	
	Therefore, it produces a stronger response for large ratio differences and emphasizes relative changes.

	\par\medskip\noindent
	\begin{minipage}{\linewidth}
  		\centering
 		\resizebox{0.9\linewidth}{!}{\includegraphics{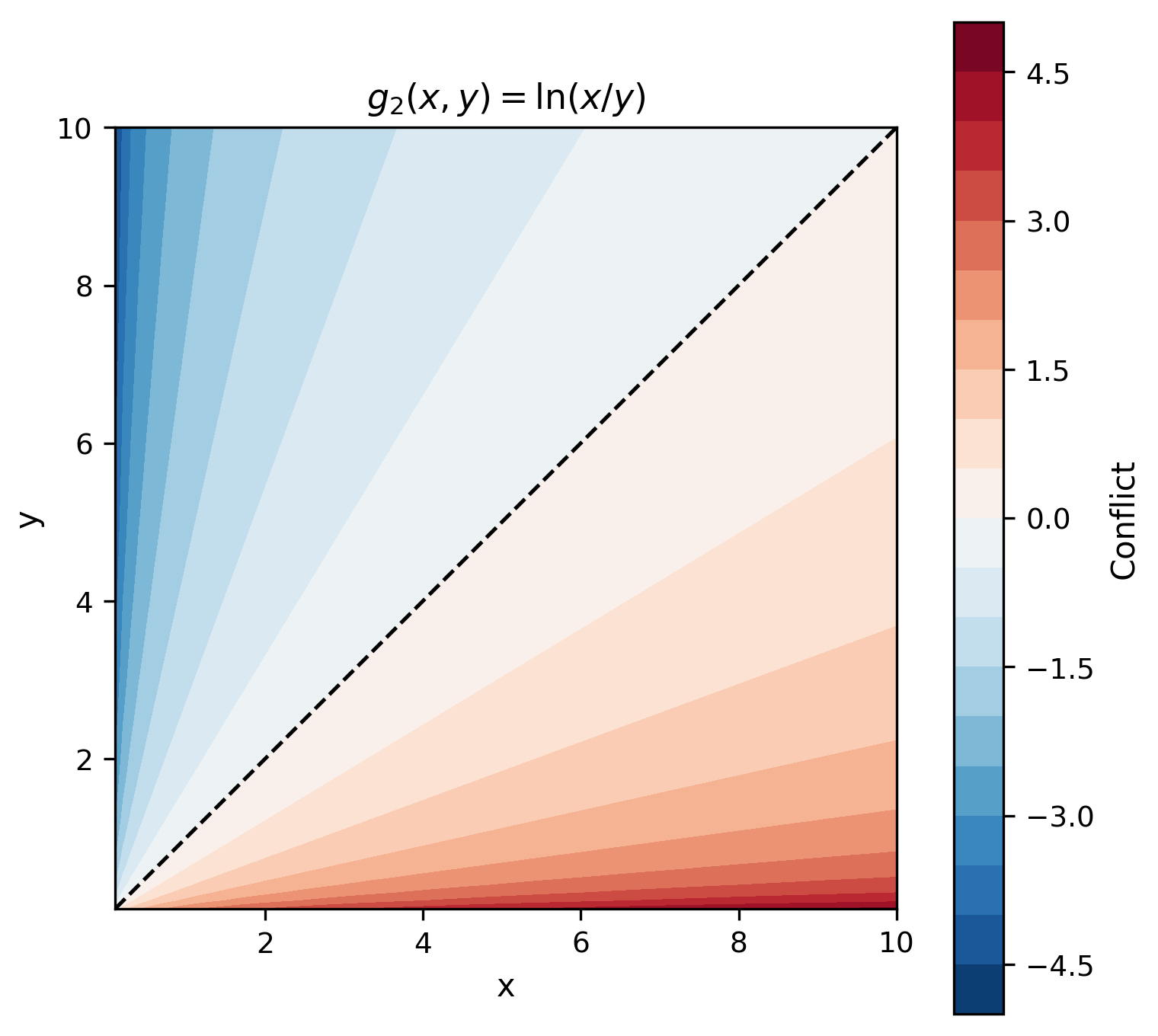}}
  		\captionof{figure}{
			Contour/heat map of the operator $g_2 = \ln \left( \frac{x}{y} \right)$.
			This operator measures conflict in a ratio-based manner and exhibits logarithmic growth behavior.
  		}
  		\label{fig:experiment-1-operators-visualization-heatmap-g2}
	\end{minipage}
	\par\medskip 

	\subsubsection*{\textbf{$g_3$ Operator}}
	\[
		g_3( x, y ) = x - y
	\]
	This operator measures conflict directly as a \textbf{raw difference}.
	It is sensitive to scale changes, and the output amplitude increases with the input magnitude.
	
	This property may be useful in applications where physical or absolute quantities carry meaning.

	\par\medskip\noindent
	\begin{minipage}{\linewidth}
  		\centering
 		\resizebox{0.9\linewidth}{!}{\includegraphics{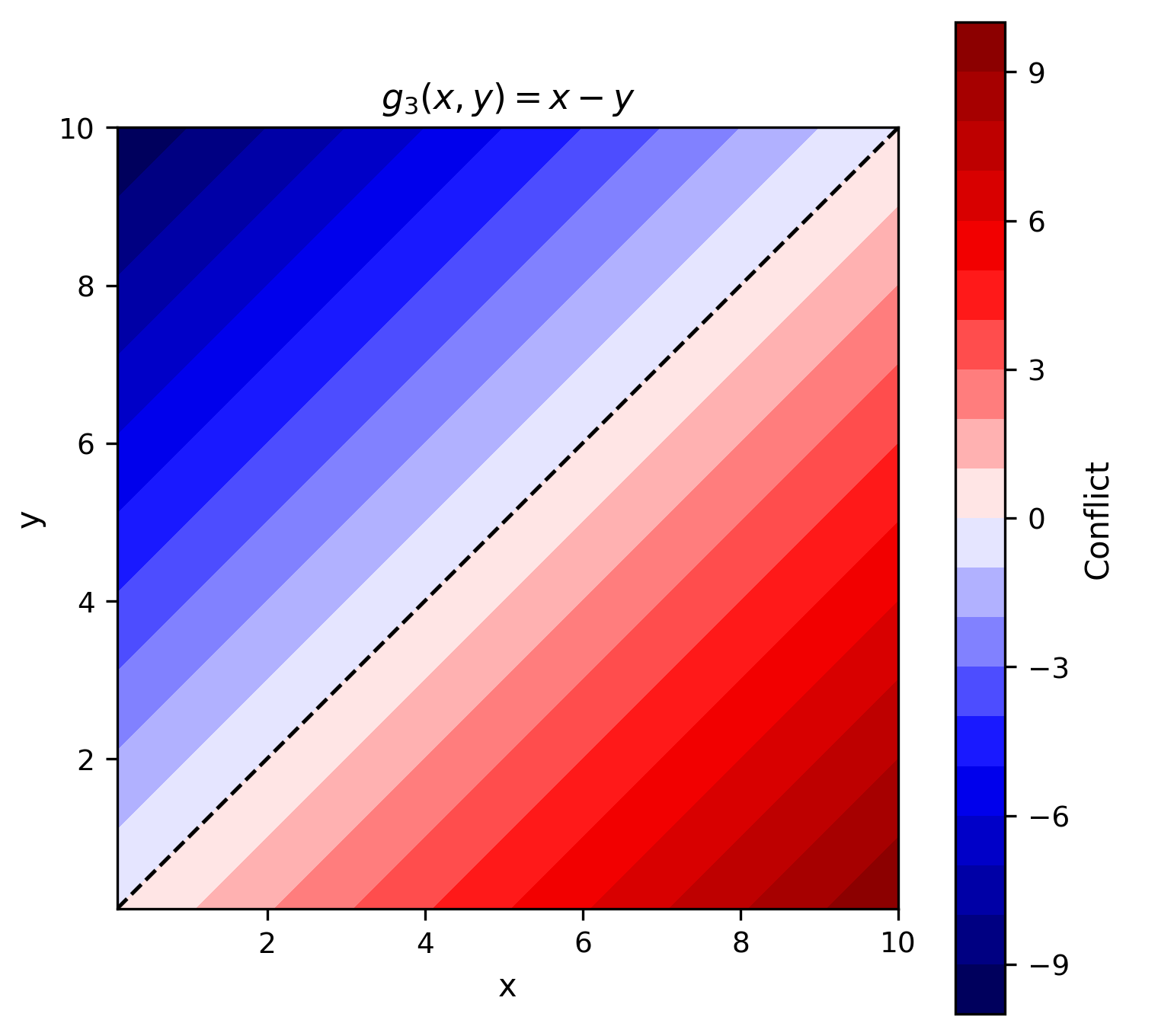}}
  		\captionof{figure}{
			Contour map of the operator $g_3( x, y ) = x - y$.
			This operator measures conflict as a raw difference and exhibits a linear structure.
  		}
  		\label{fig:experiment-1-operators-visualization-heatmap-g3}
	\end{minipage}
	\par\medskip 

	\subsection{Interpretation of the Experiment}
	The experimental results reveal three important observations:
	
	\begin{enumerate}
	
		\item All operators satisfy the \textbf{zero-conflict consistency} and \textbf{antisymmetry} axioms.
		
		\item The operators do not produce the same mathematical behavior; instead, they define
		\textbf{different conflict geometries}.
	
		\item These differences support the rationale for designing the proposed framework through an
		\textbf{operator family approach} rather than a single metric.

	\end{enumerate}
	
	These findings show that the proposed conflict framework is consistent with the theoretical structure
	and provides a flexible operator family capable of representing different interpretations of conflict.
	
	All computations, results, and visual outputs associated with this experimental verification are made
	openly accessible:
	\begin{itemize}
		\item \textbf{GitHub:} $\GitHubLink{Axiom-Level Behavioral Verification}$
		\item \textbf{Zenodo:} $\ZenodoLink$
	\end{itemize}

	\section{Experimental Validation 2: Contextual Projection with $\Csoftmax$}
	In this section, in order to demonstrate that the proposed framework is not merely an abstract formalism,
	the $\Csoftmax$ operator \cite{kartal2025csoftmax} is treated as a projection operator ($\Psi$).
	
	This choice also demonstrates that the framework can naturally integrate with context-sensitive
	decision projection mechanisms.
	
	The experiment was performed separately for conflict-free and conflict-aware projection, and the
	results were compared.
	
	The purpose of this experiment is not to propose a new softmax variant. Rather, within the proposed
	framework, the aim is to demonstrate that the modulated conflict representation can be transformed
	into a probability space through a context-sensitive projection operator.
	
	For this purpose, the raw representation matrix:
	\[
		\Raw = \begin{bmatrix}
			0.90 & 0.70 & 0.40 \\
			0.20 & 0.80 & 0.60 
		\end{bmatrix}
	\]
	and the contextual representation matrix:
	\[
		\Context = \begin{bmatrix}
			0.60 & 0.90 & 0.50 \\
			0.40 & 0.70 & 0.20 
		\end{bmatrix}
	\]
	were defined.	
	
	The $\Csoftmax$ parameters used during the projection stage were kept fixed for both experiments:
	\begin{align*}
		\vecomega & = [ 0.4, 0.3, 0.3 ] \\
		\vecbeta  & = [ 0, 0 ] \\
		\vecalpha & = [ 1, 1 ] \\
		\tau      & = 1
	\end{align*}
	
	\subsection{Conflict-Free Projection}
	In the first stage, the $\Csoftmax$ operator was directly applied to the raw representation matrix
	$\Raw$.
	That is, the operator input tensor was taken as:
	\[
		\Trxd = \Raw
	\]
	
	From this, the $\Csoftmax$ scores were computed as:
	\[
		s_i = \sum_{j = 1}^{d} \omega_j \Raw_{ij} + \beta_i
	\]
	
	As a result of the computation:
	\begin{align*}
		s_1 & = 0.4(0.90) + 0.3(0.70) + 0.3(0.40) = 0.69 \\
		s_2 & = 0.4(0.20) + 0.3(0.80) + 0.3(0.60) = 0.50 
	\end{align*}
	were obtained.
	
	When these scores were projected through the $\Csoftmax$ operator:
	\[
		p_i = \frac{
			\alpha_i \exp( s_i / \tau )
		}{
			\sum_{k = 1}^{r} \alpha_k \exp( s_k / \tau )
		}
	\]
	the resulting probability distribution:
	\[
		p_{\text{raw}} = [ 0.5474, 0.4526 ]
	\]
	was obtained.
	
	This result indicates that, within the raw representation space, the first row is more dominant than
	the second row.
	
	\subsection{Conflict-Aware Projection}
	In the second stage, the directional discrepancy between the raw representation and the contextual
	representation was computed using the proposed framework.
	
	For this purpose, the canonical conflict operator of the framework was used:
	\[
		g_1( x, y ) = \frac{x - y}{x + y}
	\]
	
	Thus, the conflict-aware representation tensor:
	\[
		\Phi^{(g_1)}( \Raw, \Context )_{ij} = \frac{
			\Raw_{ij} - \Context_{ij}
		}{
			\Raw_{ij} + \Context_{ij}
		}
	\]
	was obtained.
	
	As a result of the computation:
	\[
		\Phi^{(g_1)}( \Raw, \Context ) = \begin{bmatrix}
			0.20 & -0.125 & -0.111 \\
			-0.333 & 0.0667 & 0.50 
		\end{bmatrix}
	\]
	was obtained.
	
	Then, the row-wise normalized contextual modulation matrix:
	\[
		\Priority = \begin{bmatrix}
			0.5 & 0.3 & 0.2 \\
			0.2 & 0.3 & 0.5
		\end{bmatrix}
	\]
	was used.
	
	The modulated conflict-aware representation tensor computed via the Hadamard product:
	\[
		\Trxd^{(g_1)} = \Priority \odot \Phi^{(g_1)}( \Raw, \Context )
	\]
	was calculated as:
	\[
		\Trxd^{(g_1)} = \begin{bmatrix}
			0.10 & -0.0375 & -0.0222 \\
			-0.0666 & 0.02 & 0.25
		\end{bmatrix}
	\]
	
	This structure now constitutes a contextually modulated directional conflict representation.
	
	The obtained tensor was then projected using the same $\Csoftmax$ parameters:
	\[
		s_i = \sum_{j = 1}^{d} \omega_j \Trxd^{(g_1)}_{ij} + \beta_i
	\]
	
	Accordingly:
	\begin{align*}
		s_1 & = 0.4(0.10) + 0.3(-0.0375) + 0.3(-0.0222) = 0.0221 \\
		s_2 & = 0.4(-0.0666) + 0.3(0.02) + 0.3(0.25) = 0.0544
	\end{align*}
	score values were obtained.
	
	When these scores were projected through the $\Csoftmax$ operator:
	\[
		p_{\text{conf}} = [ 0.4919, 0.5081 ]
	\]
	the resulting probability distribution was obtained.
	
	\subsection{Comparison of Experimental Results}
	While the conflict-free projection result was:
	\[
		p_{\text{raw}} = [ 0.5474, 0.4526 ]
	\]
	the projection result obtained through the conflict-aware representation was:
	\[
		p_{\text{conf}} = [ 0.4919, 0.5081 ]
	\]
	
	In particular, the largest positive conflict value ($0.50$), corresponding to the third component
	of the second row, was weighted with high priority ($0.5$) under the contextual modulation matrix,
	thereby producing a dominant contribution within the conflict-aware representation space.
	Although the positive conflict value ($0.20$) corresponding to the first component of the first row
	was weighted with the same priority, it produced a more limited effect due to its smaller conflict
	magnitude.
	
	This demonstrates that contextual modulation can shape the geometry of conflict and thereby influence
	the decision projection process.

	Therefore, although the same $\Csoftmax$ parameters were used, the utilization of the conflict framework
	altered the decision geometry.
	The first row, which appeared more dominant in the raw representation space, lost its relative
	advantage after projection through the contextual conflict representation, and the decision geometry
	(the relative structure of the probability distribution) was reconfigured in favor of the second row.
	
	This demonstrates that the proposed framework can transfer not only data magnitudes, but also the
	directional discrepancy between raw data and contextual structure into the decision mechanism.
		
	\subsection{Conclusion}
	This experiment demonstrates that the proposed conflict framework can naturally integrate with
	context-sensitive projection operators and that the conflict-aware representation space can directly
	influence decision geometry.
	
	Thus, conflict behaves not merely as a secondary quantity emerging during optimization, but as a
	relational component capable of directly restructuring the representation space itself.
	

	\section{Discussion and Future Work}
	The conflict framework proposed in this study does not treat conflict merely as an implicit side effect
	or error term that must be optimized away; rather, it treats conflict as a directional, contextually
	modulated, and operator-based independent mathematical object.
	
	This approach structurally separates the processes of conflict measurement, contextual modulation through
	priority structures, and projection into decision/projection spaces, thereby offering a representation
	perspective distinct from classical loss-based or symmetric divergence-based approaches.
	
	The proposed structure is conceptually strongly related to representation disentanglement, contextual
	routing, explainable decision mechanisms, and modular learning architectures, all of which are becoming
	increasingly critical in modern artificial intelligence systems.
	
	In this context, conflict should not be regarded merely as ``noise'' or ``undesired optimization waste,''
	but rather as a structural signal carrying information about the contextual behavior of the system.
	
	One of the most important characteristics of the proposed framework is that it does not impose a single
	conflict metric.
	Instead, it adopts an axiomatic operator family approach representing different scale regimes.
	This approach makes it possible to handle normalized, ratio-sensitive, and absolute-scale conflict
	behaviors within the same mathematical framework. Thus, conflict analysis becomes adaptable to the nature
	of the problem domain rather than enforcing a specific application regime.
	
	The framework defined in this study, particularly when considered together with contextual projection
	operators, may provide a general mathematical infrastructure for future context-sensitive decision systems.
	
	In this direction, when considered together with the $\Csoftmax$ operator previously proposed by the
	author \cite{kartal2025csoftmax}, it becomes possible to transform the directional structural
	discrepancies produced by the conflict kernel into contextual probability projections.
	
	This may lead to the development of new projection structures capable of representing context, priority,
	and directional information within the same decision mechanism, in contrast to the classical softmax
	approach, which is primarily based on magnitude-driven competition structures.
	
	Similarly, the adaptation of the proposed conflict framework to multi-criteria decision-making (MCDM)
	problems constitutes another important research direction.
	When considered together with the $\CAHP$ and $\CAHPPP$ methods previously proposed by the author
	\cite{kartal2025cahpcahppp}, the systematic incorporation of conflict and contextual modulation into
	the decision process may allow the criterion priorities, which are generally treated as fixed and
	universal in classical AHP approaches, to be reformulated according to context, alternatives, and
	directional relational dynamics.
	
	Such an approach may pave the way for next-generation contextual decision systems that incorporate not
	only criterion weights, but also conflict geometry, contextual priority shifts, and alternative-based
	decision dynamics into the decision-making process.

	The potential application areas of the framework are not limited to decision-support systems (MCDM).
	In particular, it may open new research directions in fields such as representation learning,
	multi-task learning, attention mechanisms, adaptive weighting systems, and conflict-aware optimization
	problems.
	In addition, if the effect of context in biological systems can be modeled directionally and structurally,
	the proposed approach may also be extensible to bioinformatics.
	In particular, the fact that the same biological pattern may acquire different functional meanings under
	different contextual conditions makes context-sensitive modeling of conflict especially important.
	
	This study also opens a broader mathematical discussion concerning the long-standing ``meaning and
	context'' problems in artificial intelligence systems.
	Current large language models (LLMs) and representation systems are generally capable of statistically
	modeling contextual patterns; however, they do not explicitly and structurally represent the directional,
	priority-sensitive, and decomposed effects of context.
	Although the proposed framework does not claim to directly solve this problem, it may nevertheless offer
	a general representational perspective toward systematically incorporating context into mathematical models.
	
	Nevertheless, the present study also has several limitations.
	In particular, the proposed projection operator $\Psi$ has been intentionally kept general and was not
	restricted to a specific function family.
	This choice was made in order to increase the adaptability of the framework to different problem domains;
	however, the systematic investigation of the behavior of different projection families constitutes an
	important topic for future research.
	Similarly, the experimental validation section in this study primarily focused on demonstrating axiomatic
	behaviors and geometric properties.
	Large-scale experimental analyses on real-world datasets and application-dependent performance comparisons
	will also be addressed in future work.
	
	In conclusion, this study proposes a broader framework that treats conflict not merely as an optimization
	residue to be minimized, but as an independent mathematical representation object capable of carrying
	structural information.
	The proposed approach may contribute to establishing new theoretical connections between contextual
	structures, decision systems, representation learning, and conflict-aware artificial intelligence
	architectures.
	

	\section{Conclusion}
	In this study, a generalized and operator-based mathematical Conflict Framework ($\Gamma^{(g)}$)
	capable of systematically modeling the discrepancies between raw data and contextual data has been
	presented.
	The concept of conflict, which has traditionally been treated in the machine learning and
	multi-criteria decision-making (MCDM) literature as an implicit side effect of optimization losses,
	has been reformulated through this framework as an independent, directional, and context-sensitive
	mathematical object.

	The primary contribution of this work is the establishment of conflict analysis on an explicit
	axiomatic foundation independent of the constraints imposed by specific algorithms.
	By placing the ``Structural Decomposition Principle''
	(See Remark~\ref{rem:structural-decomposition-principle}) at the center of the framework, the
	processes of conflict measurement, directional characterization, priority modulation ($\Priority$),
	and transformation into decision space ($\Psi$) have been conceptually isolated from one another.
	This isolation prevents the tension between data representations from being prematurely reduced to a
	single scalar score, thereby avoiding irreversible information loss.

	Supported by the admissibility and positivity assumptions ($\epsilon > 0$), the proposed structure
	prevents the framework from collapsing at the definitional level, while the selected operator family
	$\Act = \{ g_1, g_2, g_3 \}$ successfully captures fundamental conflict regimes ranging from
	scale-invariant normalized measurements to information-theoretic representations and absolute-scale
	linear discrepancies.
	Furthermore, the proposed $\Gamma^{(g)}$ framework possesses linear computational complexity of
	order $O( N )$, and the differentiable nature of the selected operators enables direct and scalable
	integration into end-to-end learning architectures and gradient-based optimization processes.

	In conclusion, the presented Conflict Framework provides an adaptable, interpretable, and modular
	foundation for different classes of problems.
	This theoretical basis may pave the way for treating conflict in AI-driven decision systems not merely
	as an error to be eliminated, but as a rich and structured source of information capable of guiding
	the decision-making process.
	

\onecolumn

	\begin{appendices}
	
		\section{Canonical Formulation of the Conflict Framework}
		\label{apdx:conflict-framework-formulation}
		\begin{align*}
			\RISimplex & := \left\{ 
				\boldsymbol{X} \in \real^{r \times d}_{> 0} 
				\;\middle|\; 
				\forall{i} \in \{ 1, \dots, r \}, \; 
				\sum_{j = 1}^{d} x_{ij} = 1
			\right\}, \\[12pt]
			\epsilon & > 0, \\[12pt]
			\DataEpsilon & = \left\{
				\boldsymbol{X} \in \Data 
				\; \middle| \; 
				x_{ij} \ge \epsilon,
				\forall{i, j}
			\right\} \\[12pt]
			\Raw, \Context & \in \DataEpsilon, \\
			\Priority & \in \RISimplex, \\[12pt]
			g & : \left( \real_{> 0} \right)^{2} \to \real, \quad \text{(admissible conflict operator)}, \\ 
			\Act & := \{ g_1, g_2, g_3 \}, \\[12pt]
			g_1( x, y ) & := \frac{x - y}{ x + y }, \quad\text{([canonical] bounded, scaled-invariant)}, \\
			g_2( x, y ) & := \log\frac{x}{y}, \quad \text{(log-odds, unbounded)}, \\
			g_3( x, y ) & := x - y, \quad \text{(raw difference)}, \\[12pt]
			\RSpace^{(g)} & := \begin{cases}
				( -1, +1 )^{r \times d} & g = g_1, \\
				\Data & g = g_2, g = g_3
			\end{cases} \\[12pt]
			\Phi^{(g)} & : \left( \DataEpsilon \right)^{2} \to \RSpace^{(g)}, \\
			\Phi^{(g)}( \Raw, \Context )_{ij} & := g( \Raw_{ij}, \Context_{ij} ), \\[12pt] 
			\Psi & : \RSpace^{(g)} \to \Output, \quad \text{with} \quad \Output \in \left\{ 
				\Data, \real^{r}, \Simplex 
			\right\}, \\[12pt]
			\Gamma^{(g)} & : \left( \DataEpsilon \right)^{2} 
				\times \RISimplex 
					\to \Output \\
			\Gamma^{(g)}_{\Phi, \Psi}( \Raw, \Context, \Priority ) & := \Psi \! \left( 
				\Priority \odot \Phi^{(g)}( \Raw, \Context ) 
			\right)
		\end{align*}
	
	\end{appendices}
	
\twocolumn


	\bibliographystyle{plainnat}
	\bibliography{ConflictFramework}

@book{rudin1976principles,
  title={Principles of Mathematical Analysis},
  author={Rudin, Walter},
  year={1976},
  edition={3},
  publisher={McGraw-Hill},
  address={New York, NY, USA},
  isbn={0-07-054235-8},
  note={Section 5.19 discusses the Mean Value Theorem for real functions; multivariable extensions are covered in Chapter 9.}
}

@article{keeneyraiffa1979,
  author={Keeney, R. L. and Raiffa, H. and Rajala, David W.},
  journal={IEEE Transactions on Systems, Man, and Cybernetics}, 
  title={Decisions with Multiple Objectives: Preferences and Value Trade-Offs}, 
  year={1979},
  volume={9},
  number={7},
  pages={403-403},
  doi={10.1109/TSMC.1979.4310245}
 }

@misc{senerkoltun2019,
	title={Multi-Task Learning as Multi-Objective Optimization}, 
	author={Ozan Sener and Vladlen Koltun},
	year={2019},
	eprint={1810.04650},
	archivePrefix={arXiv},
	primaryClass={cs.LG},
	url={https://arxiv.org/abs/1810.04650}, 
}

@misc{yu2020gradientsurgerymultitasklearning,
	title={Gradient Surgery for Multi-Task Learning}, 
	author={Tianhe Yu and Saurabh Kumar and Abhishek Gupta and Sergey Levine and Karol Hausman and Chelsea Finn},
	year={2020},
	eprint={2001.06782},
	archivePrefix={arXiv},
	primaryClass={cs.LG},
	url={https://arxiv.org/abs/2001.06782}, 
}

@article{entropy,
  author={Shannon, C. E.},
  journal={The Bell System Technical Journal}, 
  title={A mathematical theory of communication}, 
  year={1948},
  volume={27},
  number={3},
  pages={379-423},
  doi={10.1002/j.1538-7305.1948.tb01338.x}
 }

@book{cover-thomas,
  author    = {Cover, Thomas M. and Thomas, Joy A.},
  title     = {Elements of Information Theory},
  publisher = {Wiley},
  edition   = {2},
  year      = {2006},
  doi			= {10.1002/047174882X}
}

@book{utility,
  author    = {Von Neumann, John and Morgenstern, Oskar},
  title     = {Theory of Games and Economic Behavior},
  publisher = {Princeton University Press},
  year      = {1944}
}

@article{luce,
  title={Individual Choice Behavior: A Theoretical Analysis.},
  author={R. Duncan Luce},
  year={1961},
  month={3},
  publisher={Taylor \& Francis, Ltd.},
  journal={Journal of American Statistical Association},
  doi={https://doi.org/10.2307/2282347},
  volume={56},
  number={293},
  pages={172-174},
  reviewed-author={Patrick Suppes}
}

@book{armstrong1988groups,
	author = {M. A. Armstrong},
	title = {Groups and Symmetry},
	publisher = {Springer New York, NY},
	doi = {https://doi.org/10.1007/978-1-4757-4034-9},
	edition = {1},
	year = {1988}
}

@book{higham2002accuracy,
   author = {Higham, Nicholas J.},
   title = {Accuracy and Stability of Numerical Algorithms},
   publisher = {Society for Industrial and Applied Mathematics},
   year = {2002},
	publisher = {SIAM},
   doi = {10.1137/1.9780898718027},
   edition   = {Second},
   URL = {https://epubs.siam.org/doi/abs/10.1137/1.9780898718027}
}

@book{feller1968introduction,
  title={An Introduction to Probability Theory and Its Applications, Vol. 1},
  author={Feller, William},
  year={1968},
  publisher={John Wiley \& Sons},
  edition={3rd},
  note={Rassal süreçlerin ve bağımsızlık varsayımının temel olasılık teorisi içinde titiz bir şekilde inşa edildiği dönüm noktası niteliğindeki eser.}
}

@book{casella2002statistical,
  title={Statistical Inference},
  author={Casella, George and Berger, Roger L.},
  year={2002},
  publisher={Wadsworth Group. Duxbury},
  isbn={0-534-24312-6},
  edition={2nd},
  note={İstatistiksel çıkarımın temellerini öğreten yaygın bir ders kitabı. Verilerin bağımsız ve aynı dağılımlı (IID) rassal bir örneklem olduğu temel varsayımı üzerine kuruludur.}
}

@article{csiszar1967information,
  author  = {Imre Csiszár},
  title   = {Information-type measures of difference of probability distributions and indirect observation},
  year    = {1967},
  journal = {Studia Scientiarum Mathematicarum Hungarica},
  volume	 = {2},
  pages	 = {299--318},
}

@article{bregman1967relaxation,
title = {The relaxation method of finding the common point of convex sets and its application to the solution of problems in convex programming},
journal = {USSR Computational Mathematics and Mathematical Physics},
volume = {7},
number = {3},
pages = {200-217},
year = {1967},
issn = {0041-5553},
doi = {https://doi.org/10.1016/0041-5553(67)90040-7},
url = {https://www.sciencedirect.com/science/article/pii/0041555367900407},
author = {L.M. Bregman}
}

@misc{bengio2013representation,
      title={Representation Learning: A Review and New Perspectives}, 
      author={Yoshua Bengio and Aaron Courville and Pascal Vincent},
      year={2014},
      eprint={1206.5538},
      archivePrefix={arXiv},
      primaryClass={cs.LG},
      url={https://arxiv.org/abs/1206.5538}, 
}

@misc{vaswani2017attention,
      title={Attention Is All You Need}, 
      author={Ashish Vaswani and Noam Shazeer and Niki Parmar and Jakob Uszkoreit and Llion Jones and Aidan N. Gomez and Lukasz Kaiser and Illia Polosukhin},
      year={2023},
      eprint={1706.03762},
      archivePrefix={arXiv},
      primaryClass={cs.CL},
      url={https://arxiv.org/abs/1706.03762}, 
}

@article{kartal2025csoftmax,
author = {Hakan Emre Kartal},
title = {C-Softmax: Contextual Softmax Operator Incorporating Row and Column Priorities},
journal = {TechRxiv},
volume = {2025},
number = {0829},
pages = {},
year = {2025},
doi = {10.36227/techrxiv.175647872.27401306/v1},
URL = {https://www.techrxiv.org/doi/abs/10.36227/techrxiv.175647872.27401306/v1},
eprint = {https://www.techrxiv.org/doi/pdf/10.36227/techrxiv.175647872.27401306/v1}}

@article{kartal2025cahpcahppp,
author = {Hakan Emre Kartal },
title = {Fusion of Quantitative and Qualitative Information in Multi-Criteria Decision Making: A Triple Information Fusion Model and an Innovative Approach in AHP},
journal = {TechRxiv},
volume = {2025},
number = {1022},
pages = {},
year = {2025},
doi = {10.36227/techrxiv.176115768.87188579/v1},
URL = {https://www.techrxiv.org/doi/abs/10.36227/techrxiv.176115768.87188579/v1},
eprint = {https://www.techrxiv.org/doi/pdf/10.36227/techrxiv.176115768.87188579/v1}}


\begin{IEEEbiography}[{\includegraphics[width=1in,height=1.25in,clip,keepaspectratio]{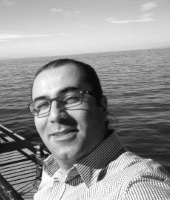}}]{Hakan Emre Kartal}
is an independent researcher working at the intersection of mathematics, artificial intelligence, and computational sciences. His research focuses on contextual decision systems, operator-based mathematical modeling, semantic weighting methods, and combinatorial representation frameworks.

He recently introduced the \emph{Contextual Softmax (C-Softmax)} operator, a context-aware probabilistic projection approach designed to incorporate external contextual structures into decision and learning systems. He is also developing semantic and operator-based frameworks for sequence representation, contextual conflict modeling, and directional discrepancy analysis, with potential applications in bioinformatics, representation learning, and AI-driven decision architectures.

His current work investigates context-sensitive mathematical operators, modular conflict representations, and scalable combinatorial methods for structured data analysis and intelligent decision systems. He also has a background in software engineering, supporting the computational and algorithmic aspects of his theoretical research.

ORCID: \href{https://orcid.org/0000-0002-3952-7235}{0000-0002-3952-7235}.
\end{IEEEbiography}

\end{document}